\documentclass[letterpaper]{article}
\usepackage{uai18}
\usepackage[margin=1in]{geometry}

\usepackage{times}

\usepackage{named}
\usepackage{titling}
\usepackage{comment}
\usepackage{ulem}

\usepackage[english]{babel}
\usepackage[utf8x]{inputenc}
\usepackage{amsmath}
\usepackage{amsthm}
\usepackage{amssymb}

\usepackage{graphicx}

\usepackage[caption=false]{subfig}

\usepackage[margin=1in]{geometry}

\usepackage{footnote}

\usepackage[toc,page]{appendix}
\usepackage[colorinlistoftodos]{todonotes}

\usepackage[colorlinks=true, allcolors=blue]{hyperref}
\usepackage{color}
\usepackage{algorithm}
\usepackage{algpseudocode}
\usepackage{multicol}

\newcommand{\G}{{\mathcal{G}}}
\newcommand{\X}{{\mathcal{X}}}
\newcommand{\ptn}{{\Delta}}

\newtheorem{defn}{Definition}
\newtheorem{thm}{Theorem}

\title{ Block-Value Symmetries in Probabilistic Graphical Models \thanks{This work is accepted at UAI 2018 [Madan et. al. 2018]}
}

\author{\bf Gagan Madan,  Ankit Anand, Mausam and Parag Singla \\
Indian Institute of Technology Delhi \\
gagan.madan1@gmail.com, \{ankit.anand, mausam, parags\}@cse.iitd.ac.in} 

\date{}

\begin{document}
\maketitle
\begin{abstract}

One popular way for lifted inference in probabilistic graphical models is to first merge symmetric states into a single cluster (orbit) and then use these for downstream inference, via variations of orbital MCMC \cite{niepert12}. These orbits are represented compactly using permutations over variables, and variable-value (VV) pairs, but they can miss several state symmetries in a domain.

We define the notion of permutations over block-value (BV) pairs, where a block is a set of variables. BV strictly generalizes VV symmetries, and can compute many more symmetries for increasing block sizes. To operationalize use of BV permutations in lifted inference, 
we describe 1) an algorithm to compute BV permutations given a block partition of the variables, 2) BV-MCMC, an extension of orbital MCMC that can sample from BV orbits, and 3) a heuristic to suggest good block partitions. 
Our experiments show that BV-MCMC can mix much faster compared to vanilla MCMC and orbital MCMC.
	
\end{abstract}
\section{INTRODUCTION}
\begin{figure*}
\centering

{\includegraphics[width=0.95\textwidth]{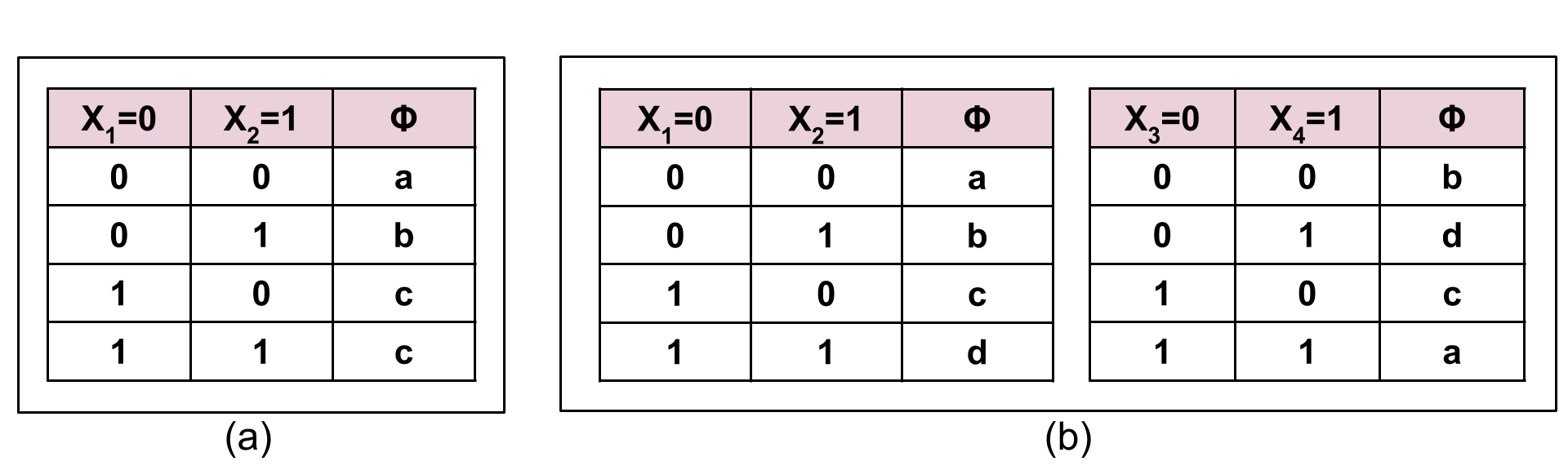}}
\caption{ Block-Value Symmetries {\bf (a)} BV Symmetries within a block {\bf (b)} BV Symmetries across blocks }
\label{fig:ex}
\vspace*{-2ex}
\label{fig:egs}
\end{figure*}
A \emph{lifted inference} algorithm for probabilistic graphical models (PGMs) performs inference on a smaller model, which is constructed by merging together states (or variables) of the original model \cite{poole03,braz&al05,kimmig&al15}. Two main kinds of lifted inference algorithms exist: those where lifting is tied to an existing inference procedure such as belief propagation \cite{singla&domingos08,kersting&al09}, Gibbs sampling \cite{venugopal&gogate12}, weighted model counting \cite{gogate&domingos11}, variational inference \cite {bui&al13} and linear programming \cite{mladenov&al12}; and those that merge symmetric states/variables independent of the procedure \cite{niepert12,broeck&niepert15,anand&al16}.

One approach for generating symmetries is by computing isomorphism over a graphical representation of the PGM. This merges symmetric states into a single cluster (orbit), which is compactly represented as permutations over a polynomial  representation. Permutations over variables \cite{niepert12} and over variable-value (VV) pairs \cite{anand&al17} 
have been studied, with latter being a generalization of the former, capturing many more state symmetries. While more general, VV permutations clearly do not capture all possible state symmetries in a domain. For example, state $s_1=(0,0,0,0)$ is symmetric to $s_2=(0,1,1,1)$ in Figure \ref{fig:egs}(b), but VV permutations cannot represent it.

A natural question arises: are there more general representations which can capture (a subset of) these larger set of symmetries? We note that the problem of computing all possible symmetries is intractable since there is an exponential number of permutations over an exponentially large state space, each of which could be a symmetry (or not).
 Nevertheless, we hope there are representations which can capture additional symmetries compared to current approaches in bounded polynomial time. More so, it would be interesting to come up with a representation that enables computation of larger and larger sets of symmetries, while paying additional costs, which could be controlled as a function of a parameter of the representation.

As a significant step toward this research question, we develop the novel notion of symmetries defined over \emph{block-value (BV) pairs}. Here, a block is a set of variables, and its value is an assignment to these variables. Intuitively, BV pairs can capture all such
VV pairs that are not permuted independently, instead, are permuted in subsets together. For example, it can capture symmetry of states $s_1$ and $s_2$ via a BV permutation which maps
$\{(X_1,0),(X_2,0)\}\leftrightarrow \{(X_3,1),(X_4,1)\}$ and $\{(X_1,0),(X_2,1)\}\leftrightarrow \{(X_3,0),(X_4,0)\}$. 

Clearly, symmetries defined over BV pairs are a strict generalization of those over VV pairs, since each VV pair is a BV pair with a block of size 1. Our blocks can be of varying sizes and the size of each block essentially controls the set of symmetries that can be captured; larger the blocks, more the symmetries, coming at an additional cost (exponential in the max size of a block). 

In this paper, we formally develop the notion of symmetries as permutations defined over a subset of BV pairs. Some of these permutations will be invalid (when blocks overlap with each other) and their application may lead to inconsistent state. In order to ensure valid permutations, we require that the blocks come from a disjoint set of blocks, referred to as a {\em block partition}. Given a block partition, we show how to compute the corresponding set of symmetries by reducing the problem to one of graph isomorphism. We also show that our BV symmetries can be thought of as VV symmetries, albeit over a transformed graphical model, where the new variables represent the blocks in the original graph. 

Next, we show that jointly considering symmetries obtained from different block partitions can result in capturing symmetries not obtainable from any single one. Since, there is an exponential number of such block partitions, we provide an efficient heuristic 
for obtaining a promising partition of blocks, referred to as a {\em candidate set}. 

Use of BV symmetries in an MCMC framework requires uniform sampling of a state from each orbit, i.e., a set of symmetric states. This turns out to be a non-trivial task when the orbits are defined over symmetries corresponding to different block partitions. In response, we design an aggregate Markov chain which samples from orbits corresponding to each (individual) candidate set in turn. We prove that our aggregate Markov chain converges to the desired distribution. As a proof of the utility of our BV symmetries, we show that their usage results in significantly faster mixing times on two different domains.  

The outline of this paper is as follows. We start with some background on variable and VV symmetries in Section \ref{sec:2}. This is followed by the exposition of our symmeteries defined over BV pairs (Section \ref{sec:3}). Section \ref{sec:mcmc} describes our algorithm for using BV symmetries in MCMC. This is followed by our heuristic to compute promising candidate sets in Section \ref{sec:5}. We present our experimental evaluation (Section \ref{sec:6}) and conclude the paper with directions for future work.

\section{BACKGROUND}
\label{sec:2}
Let $\X$ $=\{X_{1},X_{2},\hdots,X_{n}\}$ denote a set of discrete valued random variables. We will use the symbol $x_i$ to denote the value taken by the variable $X_i$. We will assume that each of the variables comes from the same domain $\mathcal{D}$. 
A state $s \in \mathcal{D}^n$ is an assignment to all the variables in the set $\mathcal{X}$. Further, $s(X_i)=x_i$ gives the value of variable $X_i$ in state $s$. We will use $\mathcal{S}$ to denote the set of all possible states. 

A Graphical Model~\cite{koller&friedman09} is a set of pairs $\{(f_j,w_j)\}_{j=1}^{m}$ where $f_j$ is a feature function defined over the variables in the set $\mathcal{X}$ and $w_j$ is its associated weight. 
\begin{defn} Action of $\theta$ on $\mathcal{G}$ results in a new graphical model where the occurrence of $X_i$ in each feature $f_j$ in $\mathcal{G}$ is replaced by $\theta(X_i)$. Given a graphical model $\mathcal{G}$, a permutation $\theta$ of the variables in ${\mathcal X}$ is said to be a {\bf variable symmetry} of ${\mathcal{G}}$ if the action of $\theta$ on ${\mathcal{G}}$ results back in ${\mathcal{G}}$.

\end{defn}
Given a state $s \in \mathcal{S}$, the action of $\theta$ on $s$, denoted by $\theta(s)$, results in a new state $s'$ such that $\forall X_i,X_j \in {\mathcal{X}}$ if 
$\theta(X_i)=X_j$ and
$s(X_j)=x_j$ then $s'(X_i)=x_j$.

The set of all variable symmetries forms a group called {\em the variable automorphic group} of $\G$ and is denoted by $\Theta$. $\Theta$ partitions the states into equivalence classes or orbits which are as defined below. 
{\bf
\begin{defn}
Given a variable automorphic group $\Theta$, the {\bf orbit} of a state $s$ under the effect of $\Theta$ is defined as $\Gamma_{\Theta}(s) =\{\theta(s)|\theta \in \Theta \}$.
\end{defn}
}

Intuitively, the orbit of a state $s$ is set of all states reachable from $s$ under the action of any permutation in the automorphic group.

We note that variable symmetries are probability preserving transformations \cite{niepert12}. Let $\mathcal{P}$ denote the distribution defined by a graphical model $\mathcal{G}$ where $\mathcal{P}(s)$ is the probability of a state $s$.

\begin{thm}
If $\Theta$ is a variable automorphic group of $\G$, then $\forall s \in \mathcal{S}$, $\forall \theta \in \Theta$, $\mathcal{P}(s)= \mathcal{P}(\theta(s))$.
\end{thm}

Anand et al. \shortcite{anand&al17} 
extend the notion of variable symmetries to those defined over variable value (VV) pairs. Let $(X_i,x_i)$ denote a VV pair and let $\mathcal{X}_V$ denote the set of all possible such pairs. Let $\phi$ denote a permutation over the set $\mathcal{X}_V$. {\em Action of $\phi$ on state} $s$, denoted by $\phi(s)${\bf ,} results in a state $s'$, such that $\forall$ $X_i,X_j \in \mathcal{X}$, if $\phi(X_i,s(X_i))=(X_j,x_j)$, then $s'(X_j)=x_j$.

There are some VV permutations which when applied to a state $s$ may result in an inconsistent state. For instance, let $\phi(X_0,0) = (X_0,0)$ , $\phi(X_1,1) = (X_0,1)$ and $s=(0,1)$, then $\phi(s)$ results in an inconsistent state with multiple values being assigned to $X_0$. Therefore, the notion of valid VV permutation needs to be defined which when applied to any state $s \in \mathcal{S}$ always results in a consistent state $s'$ \cite{anand&al17}.

\begin{defn}
A VV permutation $\phi$ over $\X_{V}$ is said to be a {\bf valid VV permutation} if whenever there exists a VV pair $(X_i,x_i)$ such that $\phi(X_i,x_i) = (X_j,x_j)$, then for all the VV pairs of the form $(X_i,x_{i}')$ where $x_{i}' \in \mathcal{D}_i$,  $\phi(X_i,x_i) = (X_j,x_{j}')$ where $x_j' \in \mathcal{D}_j$.
\end{defn}

\begin{defn}
Action of $\phi$ on $\mathcal{G}$ results in a new graphical model where the occurrence of $(X_i,x_i)$ in each feature $f_j$ in $\mathcal{G}$ is replaced by $\phi(X_i,x_i)$.
We say that $\phi$ is a  {\bf VV symmetry} of ${\mathcal{G}}$, if action of $\phi$ on ${\mathcal{G}}$ results back in ${\mathcal{G}}$. 
\end{defn}




Similar to variable symmetries, the set of all VV symmetries form a group called the VV automorphic group of $\mathcal{G}$ and is denoted by $\Phi$.
Analogously,  $\Phi$ partitions the states into orbits defined as $\Gamma_{\Phi}(s)=\{\phi(s) \vert \forall \phi \in \Phi\}$.


In the following, we will often refer to the automorphic groups $\Theta$ and $\Phi$ as symmetry groups of $\mathcal{G}$.
It can be easily seen that VV symmetries subsume variable symmetries and like variable symmetries, they are also probability preserving transformations. 
\begin{thm}
If $\Phi$ is a VV automorphic group of $\G$, then $\forall s \in \mathcal{S}$, $\forall \phi \in \Phi$, $\mathcal{P}(s)= \mathcal{P}(\phi(s))$
\end{thm}
The orbits so obtained through variable (VV) symmetries can then be exploited for faster mixing by Markov Chain Monte Carlo (MCMC) based methods as described below. 
\subsection{Orbital-MCMC}
Markov Chain Monte Carlo (MCMC) methods \cite{koller&friedman09} are one of the popular algorithms for approximate inference in Probabilistic Graphical Models. Starting with a random state, these methods set up a Markov chain over the state space whose stationary distribution is same as the desired distribution. Convergence is guaranteed in the limit of a large number of samples coming from the Markov chain.

Orbital MCMC and VV-MCMC improve MCMC methods by exploiting Variable and VV symmetries{\bf,} respectively. Given a Markov chain $\mathcal{M}$ and a symmetry group $\Phi$, starting from a sample $s_t$, any subsequent sample is obtained in 2 steps: a) An intermediate state $s'$ is obtained according to $\mathcal{M}$ b) The next sample $s_{t+1}$ is obtained by sampling a state uniformly from the orbit (Variable or VV) of the intermediate state $s'$. Sampling a state from the orbit of the intermediate state is done using the Product Replacement Algorithm \cite{celler1995,pak00}. This two step chain so obtained converges to the true stationary distribution and has been shown to have better mixing both theoretically \cite{niepert12} and empirically \cite{niepert12,anand&al17}. The key insight exploited by these algorithms is the fact that all the states in any given orbit have the same probability. 
\section{BLOCK-VALUE SYMMETRIES}
\label{sec:3}

In this section, we will present symmetries defined over blocks of variables, referred to as {\em BV Symmetries} which strictly generalize the earlier notions of symmetries defined over VV pairs. As a motivating example, Figure \ref{fig:ex} shows two Graphical Models $\G_1$ and $\G_2$. For ease of explanation these have been represented in terms of potential tables. These can easily be converted to the weighted feature representation, as defined previously. In $\G_1$, state $(1,0)$ has the same joint probability as $(1,1)$ and in $\G_2$, state $(0,0,0,0)$ has the same joint probability as $(0,1,1,1)$. However, none of these can be captured by Variable or VV symmetries.
We start with some definitions.
\begin{defn}
Let $B=\{X_1,X_2,\hdots,X_r\}$ denote a set of variables ($X_i \in \mathcal{X}$) which we will refer to as a {\bf block}. Similarly, let $b=\{x_1,x_2,\hdots,x_r\}$ denote a set of (corresponding) assignments to the variables in the block $B$. Then, we refer to the pair $(B,b)$ as a {\bf Block-Value (BV) pair}.
\end{defn}

\begin{defn}
A BV pair $(B,b)$ is said to be {\bf consistent} with a state s if $\forall X_i \in B$, $s(X_i) = x_i$ where $x_i$ is the value for variable $X_i$ in block $B$.
\end{defn}

Let ${\Delta}_{V}^{r}$ denote some subset of all possible BV pairs defined over blocks of size less than equal to $r$. For ease of notation, we will drop superscript r and denote ${\Delta}_{V}^{r}$ as $\Delta_{V}$ where r is a pre-specified constant for maximum block size. Then, we are interested in defining permutations over the elements of the set ${\Delta_V}$. Considering any set of block-value pairs in ${\Delta_V}$ and allowing permutation among them may lead to inconsistent states. Consider a graphical model defined over four variables: $\{X_1,X_2,X_3,X_4\}$. Let us consider all possible blocks of size $\leq$ 2. Then, a BV permutation permuting the singleton block $\{X_1\}$ to itself (with identity mapping on values) while at the same time, permuting the block $\{X_1,X_3\}$ to the block $\{X_2,X_4\}$ is clearly inconsistent since $X_1$'s value can not be determined uniquely. 
A natural way to avoid this inconsistency is to restrict each variable to be a part of single block while applying permutations.
Therefore, we restrict our attention to sets of blocks which are non overlapping.


\begin{defn}
Let $\Delta = \{B_1, B_2, \hdots, B_L \}$ denote a set of blocks. We define $\Delta$ to be a {\bf partition} if each variable $X_i \in \mathcal{X}$ appears in exactly one block in $\Delta$. For a partition $\Delta$, we define the {\bf block value set} $\Delta_V$ as a set of BV pairs where each block $B_l \in \Delta$ is present with all of its possible assignments. 
\end{defn}
We would now like to define permutations over the block value set $\Delta_V$, which we refer to as {\em BV-permutations}. To begin, we define the action of a BV-permutation $\psi: \Delta_V \rightarrow \Delta_V$ on a state s. The action of a BV-permutation $\psi: \Delta_V \rightarrow \Delta_V$ on a state $s$ results in a state $s'= \psi(s)$ such that $\forall (B,b) \in \Delta_V$, $(B,b)$ is consistent with $s$ if and only if $\psi(B,b)$ is consistent with $s'$

However, similar to the case of VV symmetries, any bijection from $\Delta_V \rightarrow \Delta_V$ may not always result in a consistent state. For instance, consider a graphical model with 4 variables. Let the partition $\Delta = \{(X_1, X_2), (X_3, X_4) \}$. Consider the state $s = (0, 1, 1, 0)$. In case $\psi$ is defined as $\psi(\{X_1,X_2\}, \{0,1\})$ = $(\{X_1, X_2\}, \{1,0\})$ and $\psi(\{X_3,X_4\}, \{1,0\})$ = $(\{X_1, X_2\}, \{1,1\})$, the action of $\psi$ results in an inconsistent state, since the action of $\psi$ would result in a state with $X_2$ equal to both 0 and 1 simultaneously. To address this issue, we define a BV-permutation to be valid only under certain conditions.

\begin{defn}
A BV-permutation $\psi :\Delta_V \rightarrow \Delta_V$ is said to be {\bf valid} if $\forall (B_i,b_i) \in \Delta_V,\psi(B_i, b_i) = (B_j, b_j) \Rightarrow \forall b'_i, \exists b_j'$ such that $\psi(B_i, b_i') = (B_j, b_j')$
\end{defn}

Intuitively a BV-permutation $\psi$ is valid if it maps all assignments of a block $B$ to assignments of a fixed block $B'$. 


Presently, it is tempting to define a new graphical model where each block is a multi valued variable, with domain of this variable describing all of the possible assignments. This would be useful in a lucid exposition of symmetries. To do this we must suitably transform the set of features as well to this new set of variables. Given a block partition $\Delta$, we transform the set of features $f_j$ such that for each block either all the variables in this block appear in the feature or none of them appear in the feature, while keeping all features logically invariant. We denote the set of all variables over which feature $f_j$ is defined as $\mathcal{V}(f_j)$. Further, for a block $B_l$ and a feature $f_j$, let $\bar{B}_{l} =  B_l - \mathcal{V}(f_j)$ i.e $\bar{B}_{l}$ contains the additional variables in the block which are not part of feature $f_j$. 
\begin{defn}
Given a variable $X_i$, which appears in a block $B_l \in \Delta$ and a feature $f_j$, {\bf a block consistent} representation of the feature, denoted by $f_j'$, is defined over the variables $\mathcal{V}(f_j) \cup \bar{B}_l$, such that, $f_j'({\bf x}_j, \bar{b_l}) = f_j({\bf x}_j)$ where ${\bf x}_j$, $\bar{b_l}$ denote an assignment to all the variables in $\mathcal{V}(f_j)$ and $\bar{B}_{l}$, respectively.
\end{defn}
For instance consider the feature $f = (X_2)$. Let the block $B_l$ be $\{(X_1, X_2) \}$. Then the block consistent feature $f'$ is given by $f' = (X_1 \land X_2) \lor (\neg X_1 \land X_2)$.

We extend the idea of block consistent representation to get a partition consistent representation $\hat{f}_j$. 
\begin{defn}
{\bf A partition consistent} representation of a feature $f_j$, $\hat{f}_j$ is defined by iteratively converting the feature $f_j$ to its block consistent representation for each $X_i \in \mathcal{V}(f_j)$. 
\end{defn}
 


The set of partition consistent features $\{(\hat{f_j}, w_j) \}^{m}_{j=1}$ has the property that for all $B_l \in \ptn$, $B_l \subseteq Var(\hat{f_j})$ or $B_l \cap Var(\hat{f_j}) = \phi$, i.e. all variables in each block either appear completely, or do not appear at all in any given feature. This property allows us to define a transformed graphical model $\hat{\mathcal{G}}$ over a set of multi valued variables $\mathcal{Y}$, where each variable $Y_l \in \mathcal{Y}$ represents a block $B_l \in \Delta$. The domain size of $Y_l$ is the number of possible assignments of the variables in the block $B_l$. The set of features in this new model is simply the set of transformed features $\{(\hat{f_j}, w_j)\}^{m}_{j=1}$.
As the blocks are non overlapping, such a transformation can always be carried out. 

Since the transformation of features to partition consistent features always preserves logical equivalence, it seems natural to wonder about the relationship between the graphical models $\G$ and $\hat{\G}$. We first note that each state $s$ in $\G$ can be mapped to a unique state $\hat{s}$ in $\hat{\G}$ by simply iterating over all the blocks $B_l \in \Delta$, checking which BV pair $(B_l, b_l)$ is consistent with the state $s$ and assigning the appropriate value $y_l$ to the corresponding variable $Y_l$. In a similar manner, each state $\hat{s} \in \hat{\G}$ can be mapped to a unique state in $s \in \G$.

\begin{thm}
Let $s$ denote a state in $\G$ and let $\hat{s}$ be the corresponding state in $\hat{\G}$. Then, this correspondence is probability preserving i.e., $\mathcal{P}(s) = \hat{\mathcal{P}}(\hat{s})$ where $\mathcal{P}$ and $\hat{\mathcal{P}}$ are the distributions defined by $\G$ and $\hat{\G}$, respectively. 
\end{thm}
Similar to the mapping between states, every BV-permutation $\psi$ of $\G$ corresponds to an equivalent VV-permutation $\hat{\phi}$ of $\hat{G}$ obtained by replacing each BV pair in $\G$ by the corresponding VV pair in $\hat{\G}$ (and vice-versa). Since the distributions defined by the two graphical models are equivalent, we can define BV symmetries in $\G$ as follows:
\begin{defn}
Under a given partition $\Delta$, a BV-permutation $\psi$ of a graphical model $\G$ is a {\bf BV-symmetry} of $\G$ if the corresponding permutation $\hat{\phi}$ under $\hat{G}$ is a VV-symmetry of $\hat{\G}$.
\end{defn}
We can now state the following results for BV-symmetries.
\begin{thm}
BV-symmetries are probability preserving transformations, i.e.,  for a BV-symmetry $\psi$, $\mathcal{P}(s) = \mathcal{P}(\psi(s))$ for all states $s \in \mathcal{S}$.
\end{thm}
It is easy to that the set of all BV symmetries under a given partition $\Delta$ form a group $\Psi$.
Similar to the VV orbits, we define the BV orbit of a state $s$ as $\Gamma_{\Psi}(s)=\{\psi(s) \vert \psi \in \Psi\}$. 

When the partition $\Delta$ is such that each variable appears in a block by itself, all the BV-symmetries are nothing but VV-symmetries.
\begin{thm}
Any VV-symmetry can be represented as a BV-symmetry for an appropriate choice of $\Delta$.
\end{thm}
{\textbf{Computing BV Symmetries}}

Since BV symmetry on a graphical model $\G$ is defined in terms of VV symmetry of a transformed graphical model $\hat{\G}$, BV symmetry can be trivially computed by constructing the transformed graphical model and then computing VV symmetry on $\hat{\G}$ as described by Anand et al. \shortcite{anand&al17}.
\section{AGGREGATE ORBITAL MARKOV CHAINS} \label{sec:mcmc}

Given a block partition $\Delta$, BV symmetry group $\Psi$ of $\G$ can be found by computing VV symmetry group $\Phi$ in the auxiliary graphical model $\bf{\hat{\G}}$. We further setup a Markov chain \textit{BV-MCMC($\alpha$)} over $\Psi$ to exploit BV symmetries where $\alpha \in [0,1]$ is a parameter.
\begin{defn}
Given a graphical model $\G$, a Markov chain $\mathcal{M}$ and a BV symmetry group $\Psi$, one can define {\bf a BV-MCMC($\alpha$) Markov chain $\mathcal{M}'$} as follows: From the current sample $s_t$\\
a) Sample a state $s'$ from original Markov chain $\mathcal{M}$ \\
b) i) With probability $\alpha$, sample a state $s_{t+1} = \Gamma_{\Psi}(s')$ uniformly from BV orbit of $s'$ and return $s_{t+1}$ as next sample.\\
\, \, ii) With probability $1 - \alpha$, set state $s_{t+1}=s'$ and return it as the next sample

\end{defn}

BV-MCMC($\alpha$) Markov chain is defined similar to VV-MCMC except that it takes an orbital move only with probability $\alpha$ instead of taking it always. For $\alpha=1$, it is similar to VV-MCMC, and reduces to the original Markov chain $\mathcal{M}$ for $\alpha=0$. When $\alpha=1$, sometimes, it is observed that the gain due to symmetries is overshadowed by the computational overhead of the orbital step. The parameter $\alpha$ captures a compromise between these two contradictory effects. 


\begin{thm}
Given a Graphical Model $\G$, if the original Markov chain $\mathcal{M}$ is regular, then, BV-MCMC($\alpha$) Markov chain $\mathcal{M}'$, constructed as above, is regular and converges to the unique stationary distribution of the original Markov chain $\mathcal{M}$.
\end{thm}

It should be noted that two different block partitions may capture different BV symmetries and hence may have different BV symmetry groups. In order to fully utilize all symmetries which may be present in multiple block partitions, we propose the idea of \textbf{Aggregate Orbital Markov Chain}.


Consider $K$ different block partitions $\ptn_1, \ptn_2, \hdots, \ptn_K$. We set up $K$ independent BV-MCMC($\alpha$) Markov chains, where each chain generates samples as per BV-MCMC($\alpha$) corresponding to partition $\ptn_k$. Let these chains be $\mathcal{M'}_1, \mathcal{M'}_2, \cdots, \mathcal{M'}_K$, and let the corresponding automorphism groups be $\Psi_{1}, \Psi_{2}, \hdots, \Psi_{K}$. Given an intermediate state $s'$, we would like to sample uniformly from the union of orbits $\bigcup_{k} \Psi_{k}(s')$. Since these orbits may overlap with each other, sampling a state uniformly from the union of orbits is unclear. We circumvent this problem by setting up a new Markov chain, \textbf{Aggregate Orbital Markov Chain}. This Aggregate Orbital Markov Chain utilizes all available symmetries and converges to the true stationary distribution.

\begin{defn}
Given $K$ different BV-MCMC($\alpha$) Markov chains, $\mathcal{M'}_1, \mathcal{M'}_2, \cdots, \mathcal{M'}_K$, an {\bf Aggregate Orbital Markov Chain} $\mathcal{M}^{*}$ can be constructed in the following way: Starting from state $s_t$ a) Sample a BV-MCMC($\alpha$) Markov chain $\mathcal{M'}_k$ uniformly from $\mathcal{M'}_1, \mathcal{M'}_2, \cdots, \mathcal{M'}_K$ b) Sample a state $s_{t+1}$ according to $\mathcal{M'}_k$.  
\end{defn}

\begin{thm}
 The aggregate orbital Markov chain $\mathcal{M}^{*}$ constructed from $K$ BV-MCMC($\alpha$) Markov chains, $\mathcal{M'}_1, \mathcal{M'}_2, \cdots, \mathcal{M'}_K$, all of which have stationary distribution $\pi$,  is regular and converges to the same stationary distribution $\pi$.
\end{thm}
\begin{proof}
Given each of BV-MCMC($\alpha$) Markov chains $\mathcal{M'}_k$ are regular, firstly, we prove that the aggregate Markov chain is regular. In each step of aggregate chain, one of the BV-MCMC($\alpha$) is applied and since, there is non-zero probability of returning to the same state in BV-MCMC($\alpha$) chain, there is non-zero probability of returning to the same state in $\mathcal{M}^{*}$ . Hence, aggregate chain so defined is regular and therefore, it converges to a unique stationary distribution. \cite{koller&friedman09}. \\
The only fact that remains to be shown is that the stationary distribution of $\mathcal{M}^{*}$ is  $\pi$.
Let $T^{*}(s \rightarrow s')$ represent the transition probability of going from state $s$ to $s'$ in aggregate chain $\mathcal{M}^{*}$.
We need to show that
\begin{equation}
\pi(s') = \sum_{s \in \mathcal{S}} \pi(s) * T^{*} (s \rightarrow s') 
\end{equation}
Let $T_{k}(s \rightarrow s')$ represent the transition probability of going from state $s$ to $s'$ in $\mathcal{M'}_k$
\begin{equation}{\label{eqn:1}}
\sum_{s \in \mathcal{S}} \pi(s) * T^{*} (s \rightarrow s') 
= \sum_{s \in \mathcal{S}} \pi(s) * \frac{1}{K}* \sum_{k =1}^{K} T_{k} (s \rightarrow s') 
\end{equation}
\begin{equation} {\label{eqn:2}}
 =  \frac{1}{K} \sum_{k =1}^{K} \sum_{s \in \mathcal{S}} \pi(s) * T_{k} (s \rightarrow s') 
=  \frac{1}{K} \sum_{k =1}^{K} \pi(s') = \pi(s') 
\end{equation}
Equation \ref{eqn:1} follows from the definition of aggregate chain while equation \ref{eqn:2} holds since $\mathcal{M'}_k$ converges to stationary distribution $\pi$.

\end{proof}


Aggregate Markov chain $\mathcal{M}^{*}$ so obtained not only converges to the correct stationary distribution but also results in faster mixing since it can exploit the symmetries associated with each of the individual orbital Markov chains.

\section{HEURISTICS FOR BLOCK PARTITIONS}
\label{sec:5}

We have so far computed BV symmetries \emph{given} a specific block partition. We now discuss our heuristic that suggests candidate block partitions for downstream symmetry computation (see supplementary material for pseudo-code). At a high level, our heuristic has the following two desiderata. Firstly, it ensures that there are no overlapping blocks, i.e., one variable is always in one block. Secondly, it guesses which blocks might exhibit BV-symmetries, and encourages such blocks in a partition. 

The heuristic takes the hyperparameter $r$, the maximum size of a block, as an input. It considers only those blocks (upto size $r$) in which for each variable in the block, there exists at least one other variable from the same block, such that some clause in $\mathcal{G}$ contains both of them. This prunes away blocks in which variables do not directly interact with each other, and thus are unlikely to produce symmetries. Note that these candidate blocks can have overlapping variables and hence not all can be included in a block partition.

For these candidate blocks, for each block-value pair, the heuristic computes a weight signature. The weight signature is computed by multiplying weights of all the clauses that are made true by the specific block-value assignment. The heuristic then buckets all BV pairs of the same size based on their weight signatures. The cardinality of each bucket (i.e., the number of BV pairs of the same size that have the same weight signature) is calculated and stored. 

The heuristic samples a block partition as follows. At each step it samples a bucket with probability proportional to its cardinality and once a bucket is selected, then it samples a block from that bucket uniformly at random, as long as the sampled block doesn't conflict with existing blocks in the current partition i.e., it has no variables in common with them. This process is repeated until all variables are included in the partition. In the degenerate case, if a variable can't be sampled from any block of size 2 or higher, then it gets sampled as an independent block of size 1. Once a partition is fully sampled, it is stored and the process is reset to generate another random block partition.

This heuristic encourages sampling of blocks that are part of a larger bucket in the hope that multiple blocks from the same bucket will likely yield BV symmetries in the downstream computation. At the same time, the non-conflicting condition and existence of single variable blocks jointly ensure that each sample is indeed a bona fide block partition.

\begin{table*}[] 
\centering
\begin{tabular}{|l|l|l|l|}
\hline
Domain                                                        & Rules                                                                                                                                                                                                                                                                                   & Weights                                                                               & Variables                                                                           \\ \hline
Job Search                                                    & \begin{tabular}[c]{@{}l@{}}$\forall$ x TakesML(x) $\land$ GetsJob(x)\\ $\forall$ x $\lnot$TakesML(x) $\land$ GetsJob(x)\\ $\forall$ (x,y) Connected(x,y) $\land$ TakesML(x) $\Rightarrow$ TakesML(y)\end{tabular}                                                                                           & \begin{tabular}[c]{@{}l@{}}+$w_1$\\ +$w_2$\\ $w_3$\end{tabular}                       & \begin{tabular}[c]{@{}l@{}}TakesML(x),\\ GetsJob(x), \\ Connected(x,y)\end{tabular} \\ \hline
\begin{tabular}[c]{@{}l@{}}Student \\ Curriculum\end{tabular} & \begin{tabular}[c]{@{}l@{}}$\forall$ x Maths(x) $\land$ CS(x)\\ $\forall$ x Maths(x) $\land$ $\lnot$CS(x)\\ $\forall$ x $\lnot$Maths(x) $\land$ CS(x)\\ $\forall$ x $\lnot$Maths(x) $\land$ $\lnot$CS(x)\\ $\forall$ (x,y) $\in$ Friends, Maths(x) $\Rightarrow$ Maths(y)\\ $\forall$ (x,y) $\in$ Friends, CS(x) $\Rightarrow$ CS(y)\end{tabular} & \begin{tabular}[c]{@{}l@{}}+$w_1$\\ +$w_2$\\ +$w_3$\\ +$w_4$\\ $w$\\ $w$\end{tabular} & \begin{tabular}[c]{@{}l@{}}Maths(x)\\ CS(x)\end{tabular}                            \\ \hline
\end{tabular}
\caption{Description of the two domains used in experiments. A weight of the form $+w_1$ indicates that the weight is randomly sampled for each object.}
\label{table:domains}

\end{table*}

\section{EXPERIMENTS}
\label{sec:6}

Our experiments attempt to answer two key research questions. (1) Are there realistic domains where BV symmetries exist but VV symmetries do not? (2) For such domains, how much faster can an MCMC chain mix when using BV symmetries compared to when using VV symmetries or not using any symmetries? 


\subsection{Domains}
To answer the first question, we construct two domains.  The first domain models the effect of an academic course on an individual's employability, whereas the second domain models the choices a student makes in completing their course credits. Both domains additionally model the effect of one's social network in these settings. Table \ref{table:domains} specifies the weighted first order formulas for both the domains.

\textbf{Job Search:}
In this domain, there are $N$ people on a social network, looking for a job. Given the AI hype these days, their employability is directly linked with whether they have learned machine learning (ML) or not. Each person $x$ has an option of taking the ML course, which is denoted by $TakesML(x)$. Furthermore, the variable $Connected(x,y)$ denotes whether two people $x$ and $y$ are connected in the social network or not. Finally, the variable $GetsJob(x)$ denotes whether $x$ gets employment or not. 
 
In this Markov Logic Network (MLN)\cite{domingos&lowd09}, each person $x$ participates in three kinds of formulas. The first one with weight $w_1$ indicates the (unnormalized) probability of the person getting a job and taking the ML course ($TakesML(x)\wedge GetsJob(x)$). The second formula with weight $w_2$ indicates the chance of the person getting a job while not taking the course ($\neg TakesML(x)\wedge GetsJob(x)$). Our domain assigns different weights $w_1$ and $w_2$ for each person, modeling the fact that each person may have a different capacity to learn ML, and that other factors may also determine whether they get a job or not. Finally, $x$ is more likely to take the course if their friends take the course. This is modeled by an additional formula for each pair $(x,y)$, with a fixed weight $w_3$. 

In this domain, there are hardly any VV symmetries, since every $x$ will likely have different weights. However there are {\em intra-block} BV symmetries for the block ($TakesML(x), GetsJob(x)$) for every $x$. This is because within the potential table of this block the block values (0, 0) and (1, 0) are symmetric and can be permuted.

\begin{figure*}[t!]
\centering
\subfloat[][]{
   \includegraphics[width=0.33\textwidth]{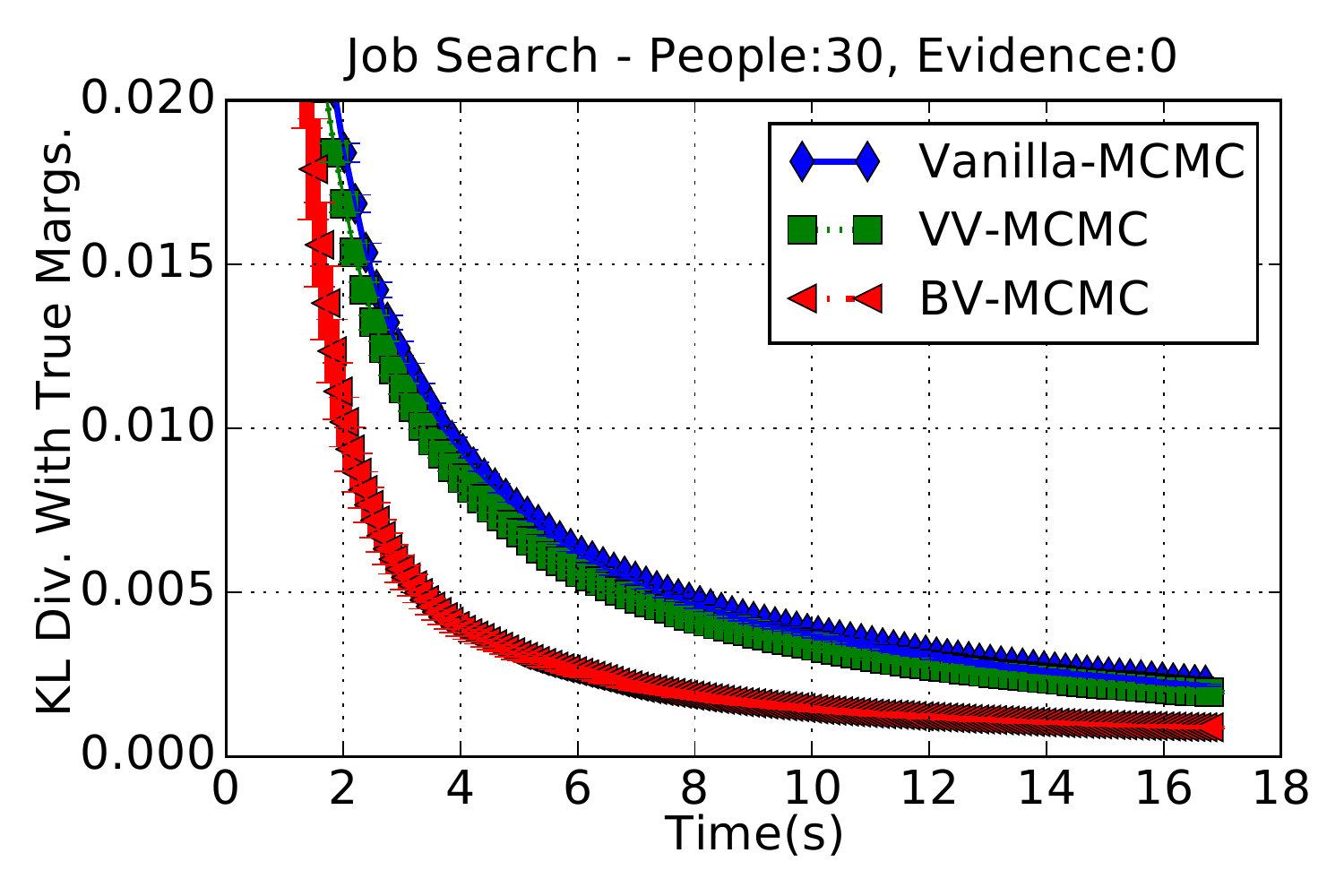}
 }
\subfloat[][]{
   \includegraphics[width=0.33\textwidth]{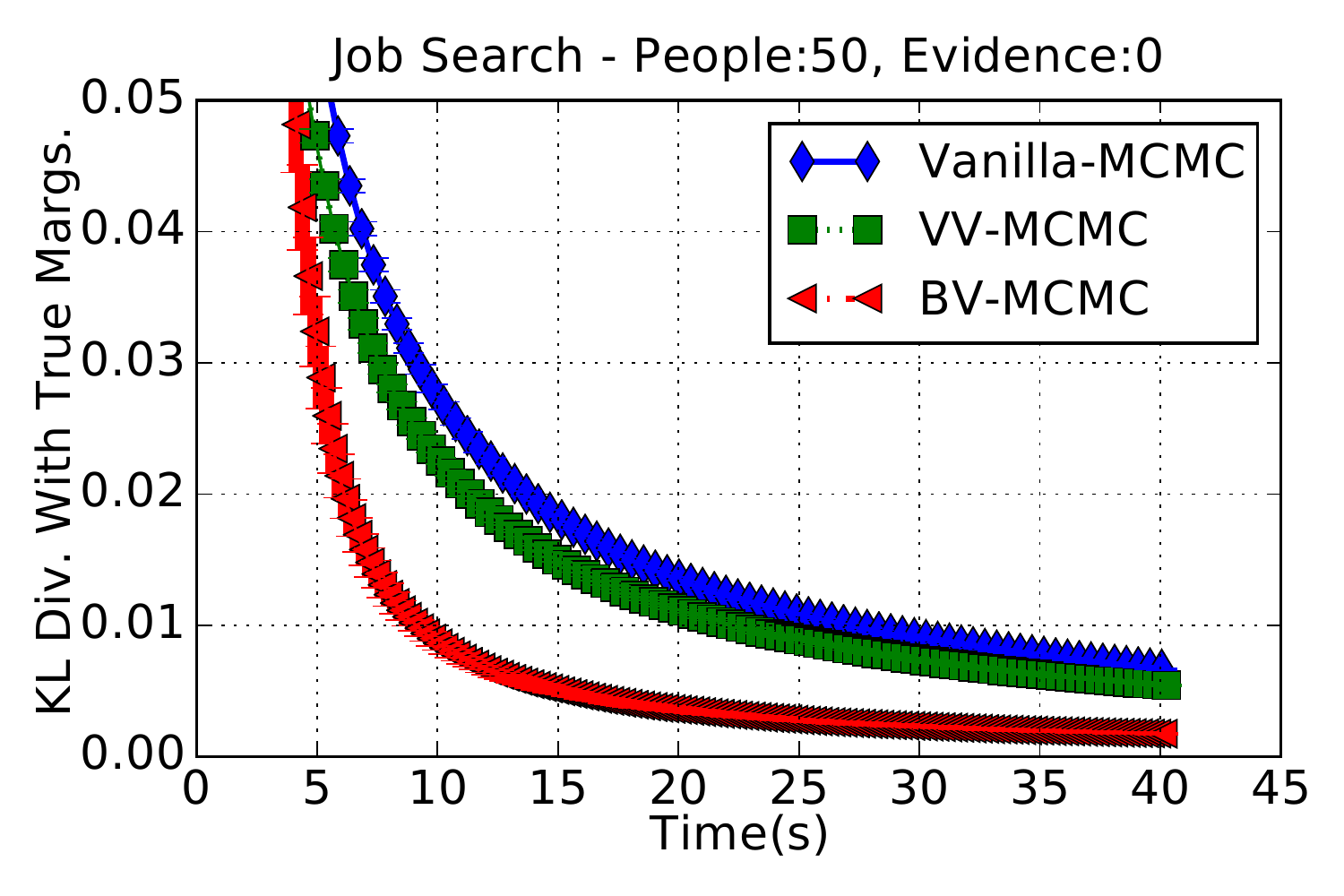}
 }
\subfloat[][]{
   \includegraphics[width=0.33\textwidth]{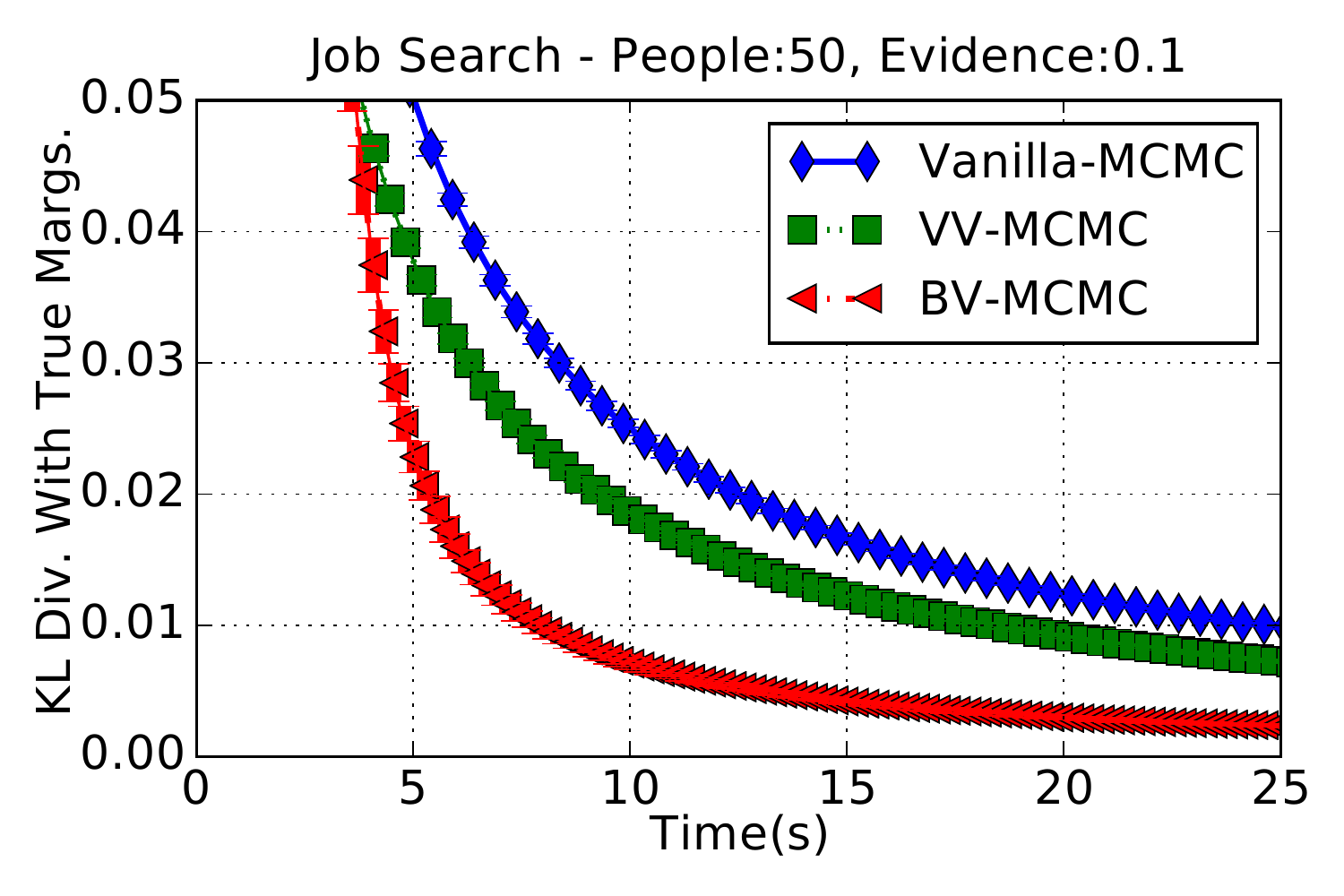}
 }

\subfloat[][]{
   \includegraphics[width=0.33\textwidth]{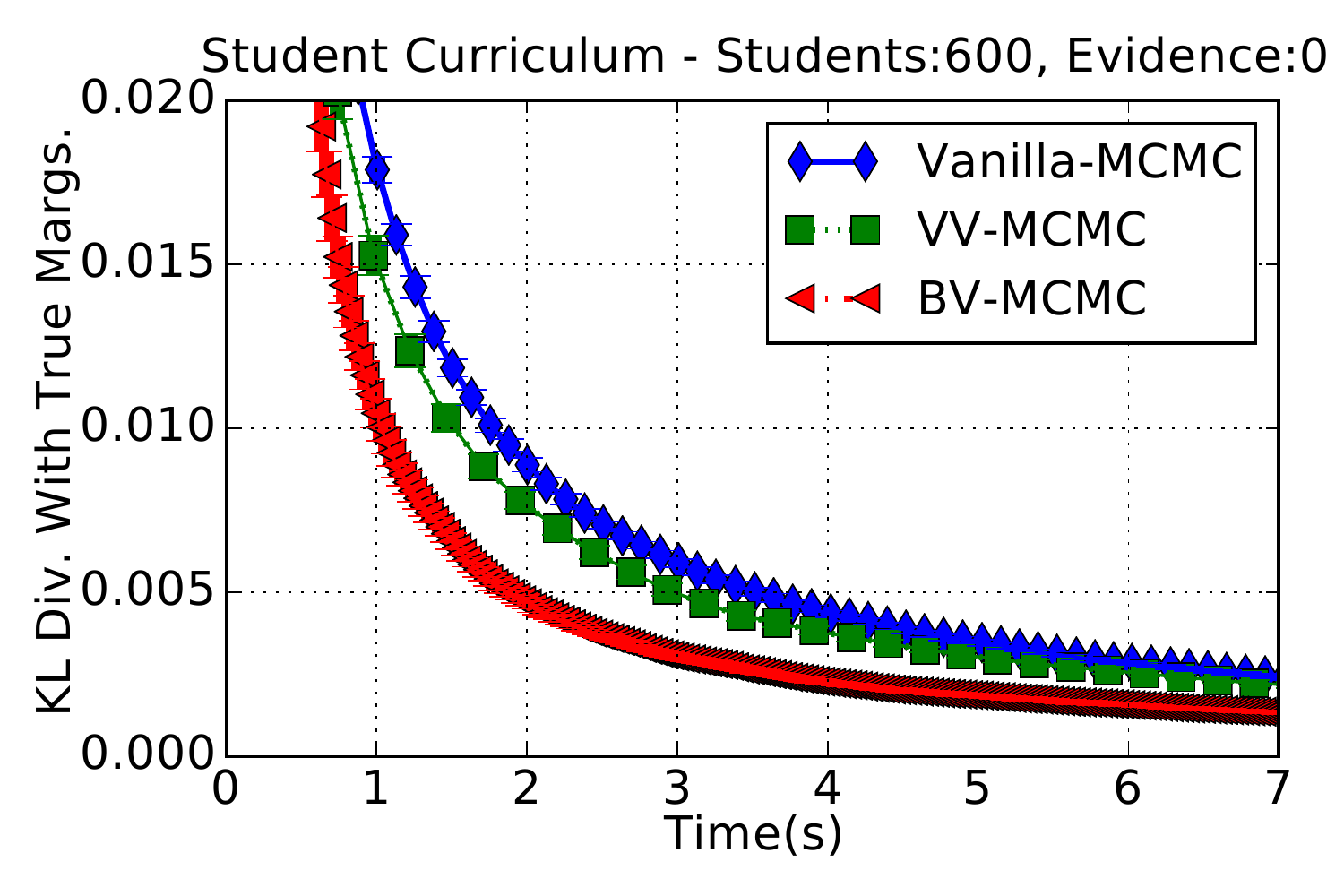}
 }
\subfloat[][]{
   \includegraphics[width=0.33\textwidth]{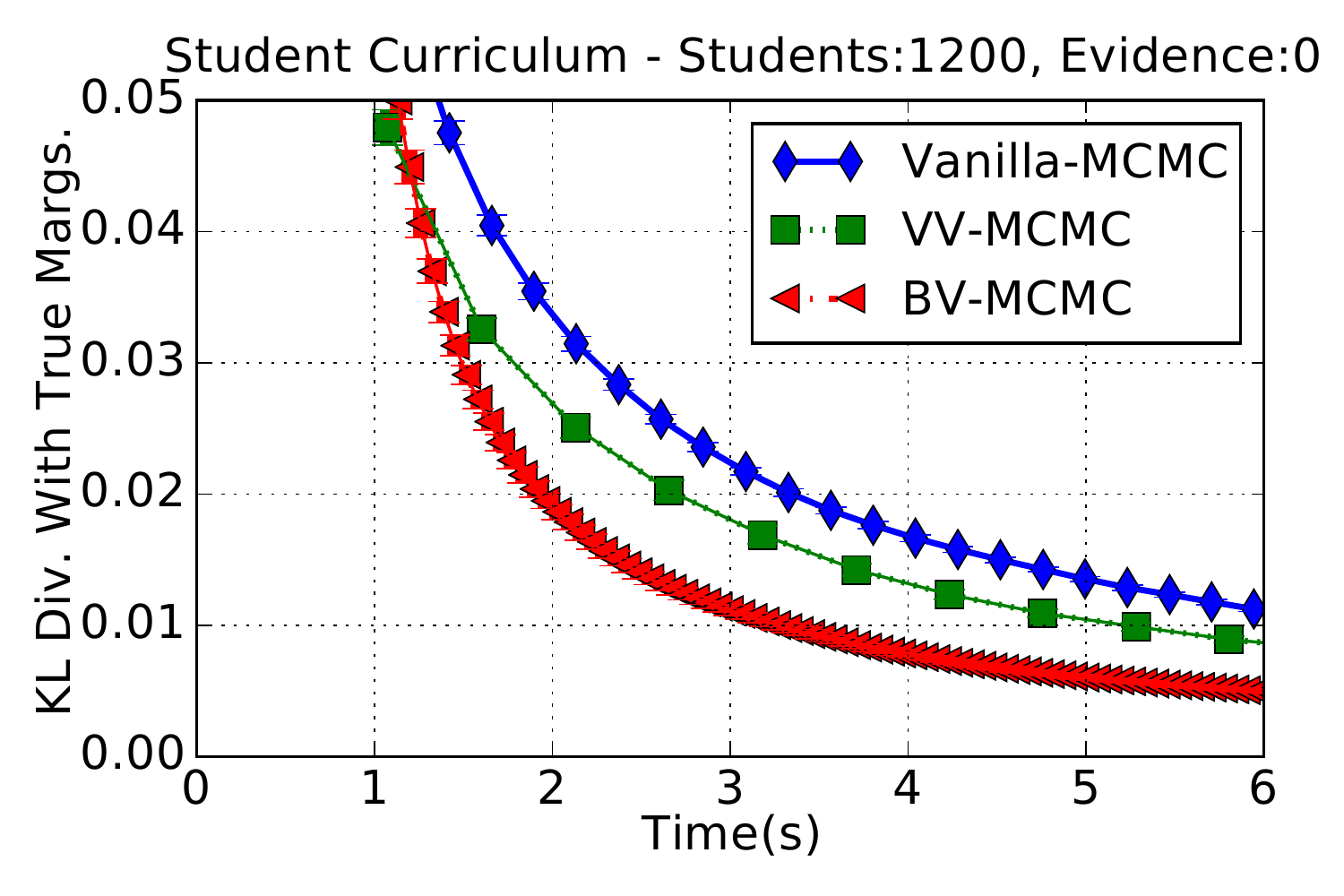}
 }
\subfloat[][]{
   \includegraphics[width=0.33\textwidth]{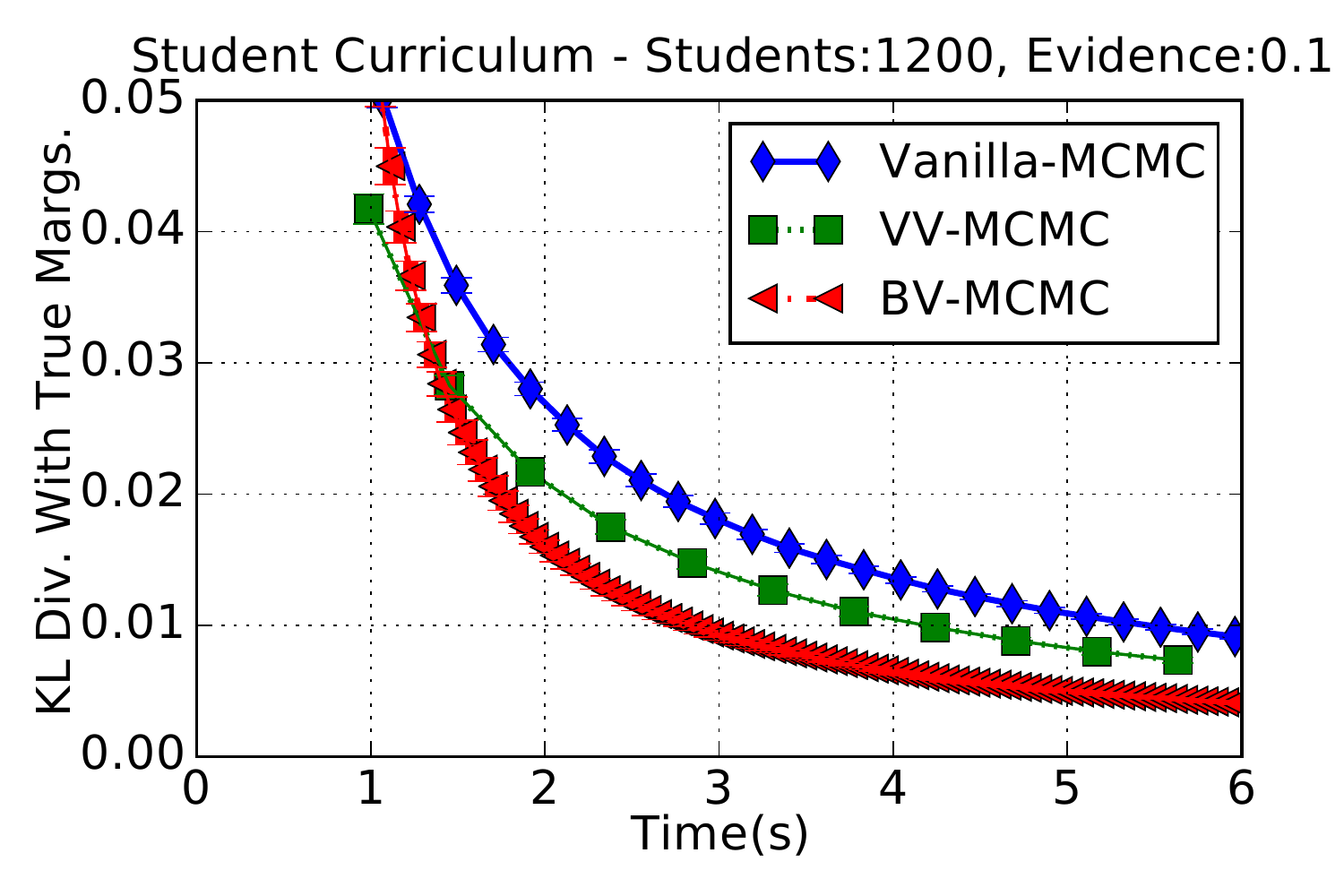}
 }

\caption{BV-MCMC($\alpha = 1$) and BV-MCMC($\alpha = 0.02$) outperforms VV-MCMC and Vanilla MCMC on Job Search and Student Curriculum domains respectively with different size and evidence variations} 
\label{fig:exp}
\end{figure*}

\textbf{Student Curriculum:} In this domain, there are $N$ students who need to register for two courses, one from Mathematics and one from Computer Science to complete their course credits. There are two courses (basic or advanced) on offer in both disciplines. Variables $Math(x)$ and $CS(x)$ denote whether the student $x$ would take the advanced course in each discipline. Since courses for Mathematics and CS could be related, each student needs to give a joint preference amongst the 4 available options. This is modeled as a potential table over ($Math(x), CS(x)$) with weights chosen randomly from a fixed set of parameters. Further, some students may also be friends. Since students are more likely to register in courses with their friends, we model this as an additional formula, which increases the probability of registering for a course in case a friend registers for the same. 

In this domain, VV pairs can only capture symmetries when the potential tables (over $Math$ and $CS$) for two students are exactly the same. 
However, there are a lot more \emph{inter-block} BV symmetries since it is more likely to find pairs of students, whose potential tables use the same set of weights, but in a different order.

\subsection{Comparison of MCMC Convergence}

We now answer our second research question by comparing the convergence of three Markov chains -- Vanilla-MCMC, VV-MCMC, and BV-MCMC($\alpha$). All three use Gibbs sampling as the base MCMC chain.
All experiments are done on Intel Core i7 machines. Following previous work, and for fair comparison, we implement all the three Markov chains in group theoretic package - GAP \cite{gap15}. This allows the use of off-the-shelf group theoretic operations. The code for generating candidate lists is written in C++. We solve graph isomorphism problems using the Saucy software \cite{saucy}. We release our implementation for future use by the community \footnote{https://github.com/dair-iitd/bv-mcmc}.



\begin{figure*}[t!]
\centering
{\includegraphics[width=0.48\textwidth]{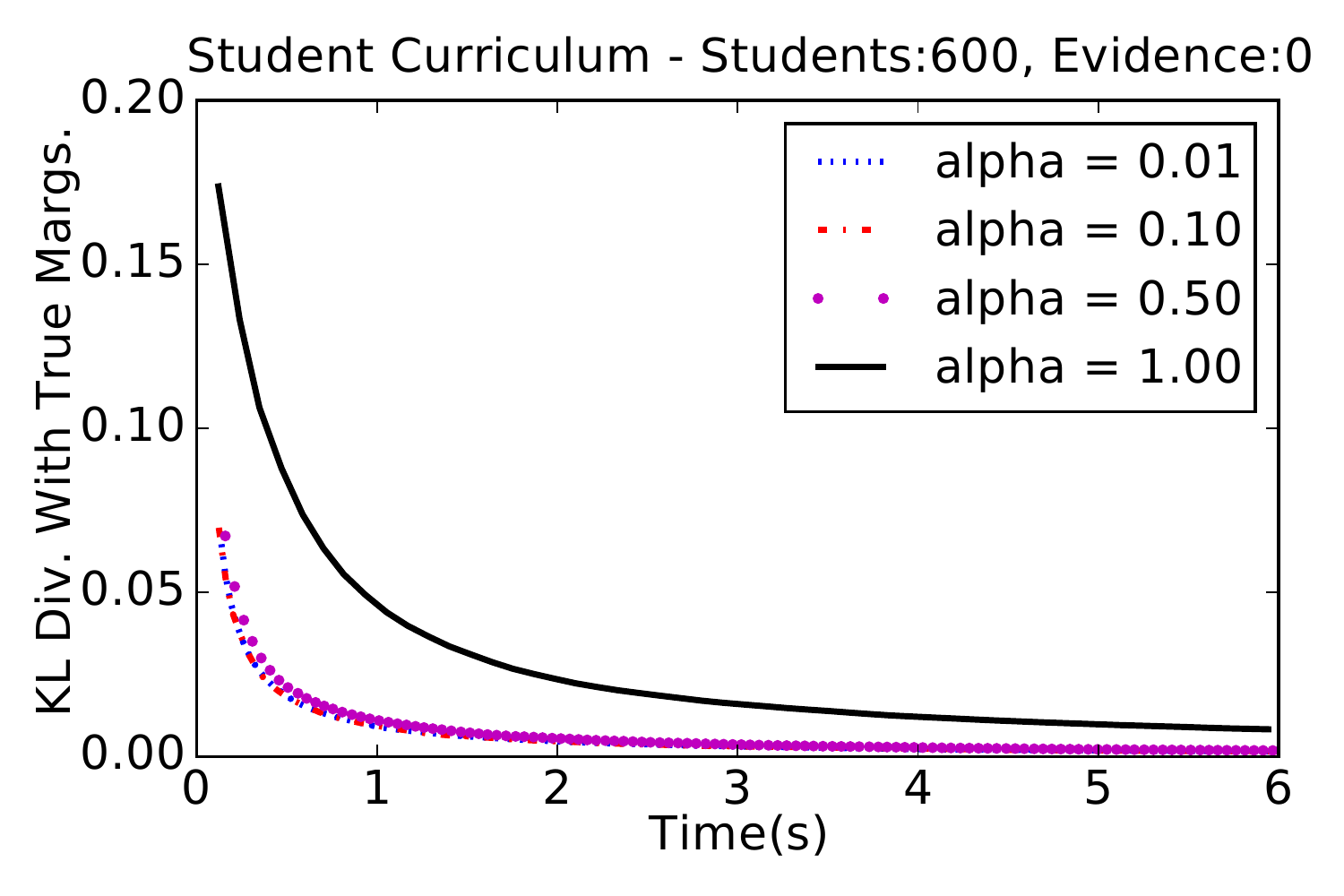}}
{\includegraphics[width=0.48\textwidth]{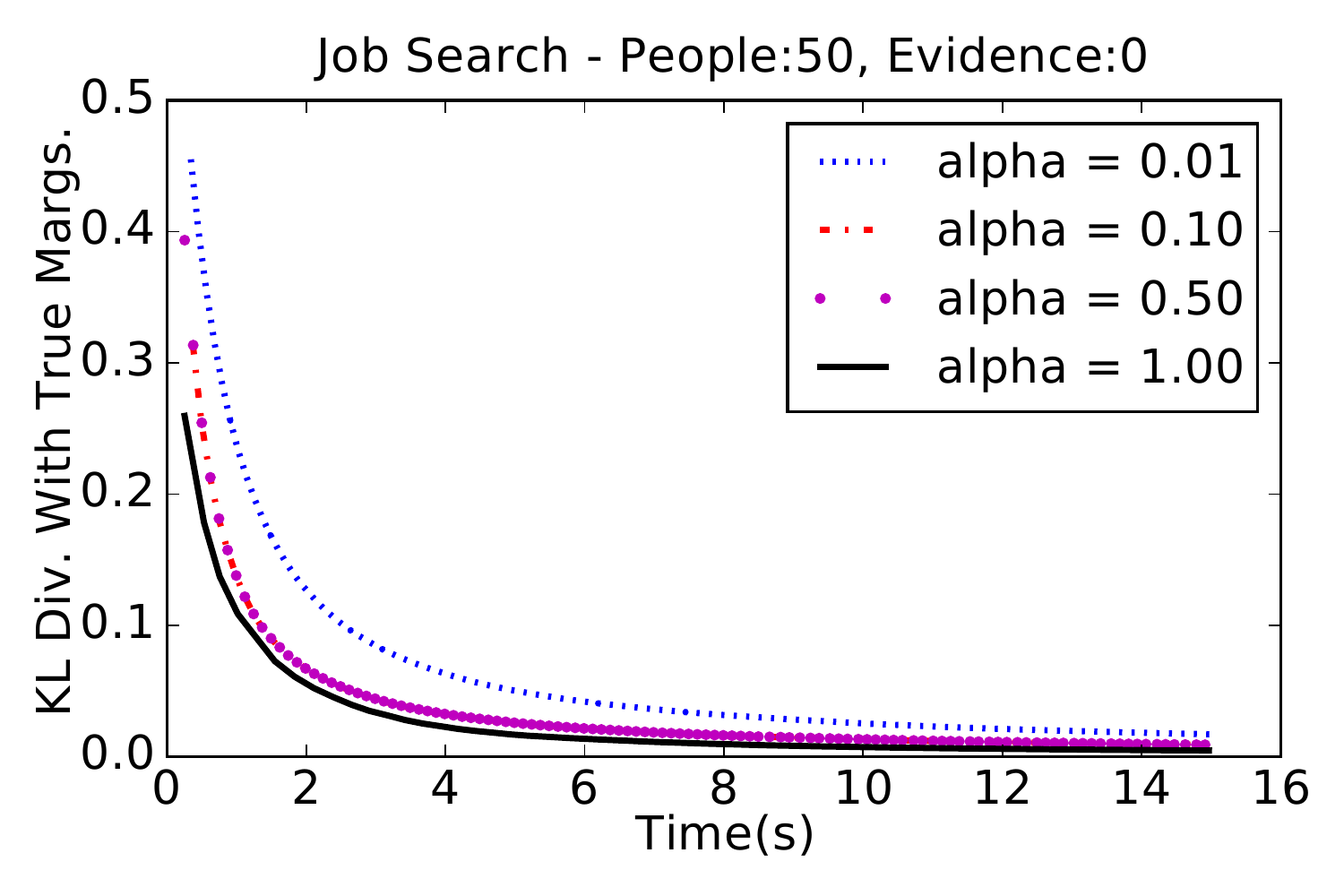}}

\caption{ Variation on $\alpha$ $\alpha < 1$ is significantly better than $\alpha=1$ in Student-Curriculum domain while $\alpha=1$ is best in Job-Search domains  } 
\label{fig:course_alpha}
\end{figure*}

In all experiments, we keep the maximum block size in a block partition to be two. For each chain we plot the KL divergence of true marginals and computed marginals for different runtimes. We estimate true marginals by running the Gibbs sampling algorithm for a sufficiently long period of time. Each algorithm is run 20 times to compute error bars indicating 95\% confidence interval. 

For VV-MCMC and BV-MCMC, the run time on x-axis includes the pre-processing time of computing symmetries as well. For BV-MCMC, this includes the time for generating candidate lists, running Saucy for each candidate list, and initializing the Product Replacement algorithm for each candidate lists. The total preprocessing time for Job Search domain is around 1.6 sec and for Student Curriculum domain is around 0.6 sec.


Figures \ref{fig:exp} shows that  BV-MCMC substantially outperforms VV-MCMC and Vanilla-MCMC in both the domains. The parameter $\alpha$ is set to 1.0 for Job Search Domain and 0.02 for Student Curriculum Domain. Since these domains do not have many VV-Symmetries, VV-MCMC only marginally outperforms Vanilla MCMC. On the other hand BV-MCMC is able to exploit a considerably larger number of symmetries and leads to faster mixing. BV-MCMC scales well with domain size, significantly outperforming other algorithms as domain size is changed from 30 to 50 people in Job Search and 600 to 1200 in Student Curriculum domain. This is particularly due to more symmetries being captured by BV-MCMC for larger domain sizes. \footnote{Most of the error-bars are negligible in size.}

Figure \ref{fig:exp}(c) and \ref{fig:exp}(f) plot the variation with introduction of 10\% evidence in each domain. BV MCMC still outperforms VV-MCMC and Vanilla-MCMC and is robust to presence of evidence. 

Finally, we also test the sensitivity of BV-MCMC with the $\alpha$ parameter. Figure \ref{fig:course_alpha} plots this variation on both these domains. We find that for Job Search, a high value $\alpha=1$ performs the best, whereas a lower value is better in Student Curriculum. This is because Job Search mostly has intra-block BV symmetries, which can be computed and applied efficiently. This makes sampling an orbital step rather efficient. On the other hand, for Student Curriculum, the inter-block symmetry between different pairs of people makes the orbital step costlier, and reducing the fraction of times an orbital move is taken improves the overall performance.

\section{CONCLUSIONS}
\label{sec:7}

Permutations defined over variables or variable-value (VV) pairs miss a significant fraction of state symmetries. We define permutations over block-value (BV) pairs, which enable a subset of variables (block) and their assignment to jointly permute to another subset. This representation is exponential in the size of the maximum block $r$, but captures more and more state symmetries with increasing $r$.

Novel challenges arise when building the framework and algorithms for  BV permutations. First, we recognize that all BV permutations do not lead to valid state symmetries. For soundness, we impose a sufficient condition that each BV permutation must be defined on blocks with non-overlapping variables. Second, to compute BV symmetries, we describe a graph-isomorphism based solution. But, this solution expects a block partition as an input, and we cannot run it over all possible block partitions as they are exponential in number. In response, we provide a heuristic that outputs candidate block partitions, which will likely lead to BV symmetries. Finally, since the orbits from different block partitions may have overlapping variables, they cannot be explicitly composed in compact form. This makes it difficult to uniformly sample from the aggregate orbit (aggregated over all block partitions). To solve this challenge, we modify the Orbital MCMC algorithm so that in the orbital step, it uniformly samples from the orbit from any one of the block partitions (BV-MCMC). We prove that this aggregate Markov chain also converges to the true posterior.

Our experiments show that there exist domains in which BV symmetries exist but VV symmetries may not. We find that BV-MCMC mixes much more rapidly than base MCMC or VV-MCMC, due to the additional mixing from orbital BV moves. 
Overall, our work provides a unified representation for existing research on permutation groups for state symmetries. In the future, we wish to extend this notion to approximate symmetries, so that they can be helpful in many more realistic domains as done in earlier works \cite{habeeb&al17}.

\section*{ACKNOWLEDGEMENTS}

We thank anonymous reviewers for their comments and suggestions and Happy Mittal for useful discussions. Ankit Anand is supported by the TCS Fellowship. Mausam is supported by grants from Google and Bloomberg. 
Parag Singla is supported by the DARPA Explainable Artificial Intelligence (XAI) Program with number N66001-17-2-4032. Both Mausam and Parag Singla are supported by the Visvesvaraya Young Faculty Fellowships by Govt. of India and IBM SUR awards. Any opinions, findings, conclusions or recommendations expressed in this paper are those of the authors and do not necessarily reflect the views or official policies, either expressed or implied, of the funding agencies.

\nocite{niepert&broeck14}
\nocite{pak00}
\nocite{madan&al18}
\cleardoublepage
\bibliographystyle{named}
\bibliography{references}
\newpage
\title{}
\title{Supplementary Material: \\ Block-Value Symmetries in Probabilistic Graphical Models $^{*}$}
\date{}
\maketitle

\section*{Algorithmic and Implementation Details for Finding Block Partitions}
This section provides algorithmic and implementation details for the heuristic used to find the candidate set of block partitions (Section 5). 
There are three broad steps for obtaining a good candidate set and each one of them is described below in turn.

Algorithm \ref{alg:getUseful}: Procedure $Get\_Useful\_Blocks$ takes a parameter $r$ and computes potentially useful blocks with maximum block-size $r$. It iterates over each of the features in turn and selects all possible subsets of size $\leq r$ of variables which are part of that feature (lines 2-7). This automatically eliminates all $r$ (or less) sized blocks which are composed of variables that never appear together in any feature in the graphical model.
\begin{algorithm}
\caption{Get\_Useful\_Blocks($\G, r$)}\label{alg:getUseful}
\begin{algorithmic}[1]
\State useful\_blocks $\gets$ \{\}
\For{f $\in$ features($\mathcal{G}$)}
	\ForAll{b $\subset$ Var(f) and Size(b)$\leq$ r}
		\State useful\_blocks $\gets$ useful\_blocks $\cup$ b
	\EndFor
\EndFor
\State \Return useful\_blocks
\end{algorithmic}
\end{algorithm}

Algorithm \ref{alg:getColour}: For the useful blocks obtained above, our heuristic constructs a weight signature for each of the block-value pairs. 
Procedure $Get\_Weight\_Sign$ computes a weight signature for all the features consistent with the input BV pair($B,b$). We define the Feature Blanket of a variable $X_j \in B$ as the set of features in which $X_j$ appears. In line 1, we construct feature blanket of a block $B$ by taking union of the feature blankets of all the variables appearing in the block. Line 2 initializes the signature as an empty multi-set. We construct weight signature by iterating over features present in feature blanket of this block. For each feature $f_j$, we check whether the given BV pair ($B,b$) is consistent with $f_j$, i.e., whether the feature is satisfied by the block-value pair.
The weight of $f_j$ is inserted in the signature if the consistency requirement is met (line 5). The complete weight-signature so obtained after iterating over all the features in the blanket is returned as the weight-signature 
for the BV pair ($B,b$).

\begin{algorithm}
\caption{Get\_Weight\_Sign($\G, B, b$)}\label{alg:getColour}
\begin{algorithmic}[1]
\State feature\_blanket($B$) $\gets \bigcup\limits_{X_j \in B}$ FeatureBlanket($X_j$)
\State signature $\gets$ \{\}
\For{f $\in$ feature\_blanket}
	\If{$(B,b)$ is consistent with f}
    	\State Insert weight(f) in signature
    \EndIf
\EndFor
\State \Return signature
\end{algorithmic}
\end{algorithm}

Algorithm \ref{alg:getBP}: This makes use of the two procedures described above and outlines the complete process for generating multiple block partitions. It takes as input a Graphical Model $\mathcal{G}$ and maximum block-size $r$. After obtaining useful blocks, a weight signature dictionary is constructed with key as weight-signature and value as a list of blocks. For each block $B_i \in useful\_blocks$, we iterate over all value assignments of that block ($\mathcal{V}(B_i)$) to form all possible BV pairs (lines 3,4). For each BV pair, Procedure $Get\_Weight\_Sign$ computes the weight-signature $S_j$ for that BV pair (line 5). If $S_j$ has already been seen in dictionary, the current block is appended to the list of blocks corresponding to the signature (lines 6,7). Else, a new weight-signature along with the list of singleton  block is added to the dictionary (lines 8-10). 

Once the weight-signature dictionary is built, we generate useful candidate lists by picking blocks using the weight signature dictionary (loop at line 14). Line 15 initializes an empty candidate list. Blocks are added to the candidate list in iterative fashion until all the variables are included (line 16). A two step sampling procedure is used. The first step samples a weight-signature with a probability proportional to the size of its corresponding list of blocks (line 17). The second step samples a block uniformly from the list of blocks 
sampled in the first step. The sampled block is added to the current candidate list if it does not overlap with pre-existing blocks (lines 19-21) otherwise a new block is sampled as above. Once all variables are added, the candidate list is complete and the process is run again till $max\_candidate\_list$ number of lists are generated.   

\begin{algorithm}
\caption{Generate\_Block\_Partitions($\G, r$)}\label{alg:getBP}
\begin{algorithmic}[1]
\State useful\_blocks $\gets $ Get\_Useful\_Blocks($\G, r$)
\State $Weight\_Sign\_Dict \gets \{\}$ 
\ForAll{$B_i \in useful\_blocks$}
	\ForAll{$b_i \in \mathcal{V}(B_i)$}
    	\State $S_i \gets$ Get\_Weight\_Sign($\mathcal{G},B_i, b_i$)
        \If{$S_i \in Weight\_Sign\_Dict$}
        	\State Append $B_i$ to $Weight\_Sign\_Dict[S_i]$
        \Else
        	\State Insert $[S_i, [B_i]]$ to $Weight\_Sign\_Dict$ 
        \EndIf
	\EndFor
\EndFor
\State $Candidate\_List \gets$ [ ]
\For{$i \gets 1 \to$ max\_candidate\_lists}
	\State $CL \gets$ \{ \} 
	\While{All variables not included in $CL$}
        
		\State Sample $S_j$ with probability $\propto$   
 $|Weight\_Sign\_Dict[S_j]|$
        \State Sample a block $b$ uniformly from $Weight\_Sign\_Dict[S_j]$
        
        \If{Variables($b$) $\cap$ $CL$ = $\phi$}
        	\State $CL \gets CL \cup b$   
        \EndIf
    \EndWhile
    \State $Candidate\_List \gets Candidate\_List \cup CL$
\EndFor
\State \Return{$Candidate\_List$} 
\end{algorithmic}
\end{algorithm}


\end{document}